\documentclass[letterpaper, 10 pt, journal, twoside]{ieeetran}
                                                 
\title{Predictive Control with Indirect Adaptive Laws for Payload Transportation by Quadrupedal Robots}

\usepackage{amsmath,amssymb,amsfonts}
\usepackage{cite}
\usepackage{graphicx}
\usepackage{epsfig} 
\usepackage{mathrsfs}
\usepackage{xcolor}
\usepackage{balance}
\usepackage{multicol}
\usepackage{array}
\usepackage{multirow}
\usepackage{booktabs}
\usepackage{epsfig} 

\newtheorem{theorem}{\textbf{Theorem}}
\newtheorem{lemma}{\textbf{Lemma}}

\newtheorem{remark}{\textbf{Remark}}

\newcommand{\Real}{\mathbb R}

\newcommand{\col}{\textrm{col}}

\newcommand{\des}{\textrm{des}}
\newcommand{\skews}{\mathbb S}
\newcommand{\identity}{\mathbb I}

\newcommand{\rom}{\textnormal}

\newcommand{\Integer}{\mathbb{Z}_{\geq0}}

\author{Leila~Amanzadeh$^{1}$, Taizoon~Chunawala$^{1}$, Randall~T.~Fawcett$^{2}$, Alexander~Leonessa$^{1}$, and Kaveh~Akbari~Hamed$^{1}$
 
\thanks{Published in IEEE Robotics and Automation Letters, vol. 9, 
no. 11, pp. 10359--10366, Nov. 2024. DOI: 10.1109/LRA.2024.3474550\\
Manuscript received: June 6, 2024; Revised August 29, 2024; Accepted September 24, 2024. This paper was recommended for publication by Editor Abderrahmane Kheddar upon evaluation of the Associate Editor and Reviewers' comments. The work of L. Amanzadeh is supported by the National Science Foundation (NSF) under Grant 1923216. The work of K. Hamed is supported by the NSF under Grants 2024772 and 2423725.}
\thanks{$^{1}$L. Amanzadeh, T. Chunawala, A. Leonessa, and   K. Akbari Hamed (\textit{Corresponding  Author})  are  with  the  Department  of  Mechanical  Engineering,  Virginia   Tech,   Blacksburg,   VA   24061, USA, {\tt\small \{leila7, taizoonac, aleoness, kavehakbarihamed\}@vt.edu}}%
\thanks{$^{2}$R. T. Fawcett is with Exponent Inc., Phoenix, AZ 85027, USA, {\tt\small randallf@vt.edu}}%
}

\begin{document}
\maketitle     

\begin{abstract}
This paper formally develops a novel hierarchical planning and control framework for robust payload transportation by quadrupedal robots, integrating a model predictive control (MPC) algorithm with a gradient-descent-based adaptive updating law. At the framework's high level, an indirect adaptive law estimates the unknown parameters of the reduced-order (template) locomotion model under varying payloads. These estimated parameters feed into an MPC algorithm for real-time trajectory planning, incorporating a convex stability criterion within the MPC constraints to ensure the stability of the template model's estimation error. The optimal reduced-order trajectories generated by the high-level adaptive MPC (AMPC) are then passed to a low-level nonlinear whole-body controller (WBC) for tracking. Extensive numerical investigations validate the framework's capabilities, showcasing the robot's proficiency in transporting unmodeled, unknown static payloads up to $109\%$ in experiments on flat terrains and $91\%$ on rough experimental terrains. The robot also successfully manages dynamic payloads with $73\%$ of its mass on rough terrains. Performance comparisons with a normal MPC and an $\mathcal{L}_{1}$ MPC indicate a significant improvement. Furthermore, comprehensive hardware experiments conducted in indoor and outdoor environments confirm the method’s efficacy on rough terrains despite uncertainties such as payload variations, push disturbances, and obstacles.
\end{abstract}



\vspace{-1em}
\section{Introduction}
\label{Introduction}
\vspace{-0.2em}

\IEEEPARstart{H}{uman}-centered environments, such as factories, offices, and homes, are designed for humans who can easily navigate gaps and uneven terrains. This has driven the development of agile-legged robots capable of traversing these environments to assist humans in various tasks, including transporting payloads. However, developing real-time trajectory planning and control algorithms for quadrupedal robots carrying unknown payloads presents significant challenges due to their nonlinear, hybrid, and high-dimensional locomotion models. Although important theoretical and technological advances have facilitated the creation of adaptive controllers for uncertain dynamical systems \cite{slotine1987adaptive,lavretsky2012robust,Nonlinear_Adaptive_Control_Design,L1_Adaptive_Control_Theory}, these approaches may face limitations when applied to full-order locomotion models due to their high dimensionality, underactuation, and unilateral constraints. 

Reduced-order models, also known as templates \cite{Full_Koditschek_Template}, offer low-dimensional representations for nonlinear and complex locomotion models. These templates have been successfully integrated with model predictive control (MPC) algorithms for real-time trajectory optimization of legged robots \cite{KF_mitcheetah3,Wensing_VBL_HJB,Abhishek_Hae-Won_TRO,grandia2019frequency,Hamed_Kim_Pandala_MPC,pandala2022robust}. Commonly used reduced-order models include the linear inverted pendulum (LIP) model \cite{kajita19991LIP}, centroidal dynamics \cite{orin2013centroidal,centroidal_dyn_ATLAS}, and single rigid body (SRB) dynamics \cite{KF_mitcheetah3,Wensing_VBL_HJB,Abhishek_Hae-Won_TRO,pandala2022robust,NMPC_CBF_Ames_Hutter}. However, these MPC algorithms typically rely on prior knowledge of template models' mass and moment of inertia and do not integrate adaptive laws to estimate unknown parameters that arise during payload transportation. This limitation can hinder the adaptability and robustness of the control system for payload transportation in real-world scenarios. 

The \textit{overarching goal} of this paper is to rigorously establish a hierarchical planning and control framework tailored for robust payload transportation using quadrupedal robots. The proposed framework seamlessly integrates an MPC algorithm with a gradient-descent-based adaptive law. The adaptive law estimates the unknown parameters of the template model, especially when carrying varying payloads. These estimated parameters are subsequently integrated into the MPC algorithm, which incorporates a convex stability criterion within its constraints for the estimation dynamics. The paper validates the efficacy of the proposed framework via extensive numerical simulations and experiments subject to various uncertainties, including payloads, external disturbances, and rough terrains (see Fig. \ref{Fig:MainSnap}). Our results demonstrate a significant improvement in performance with the proposed adaptive MPC (AMPC) algorithm compared to both normal and $\mathcal{L}_{1}$ MPCs.

\begin{figure}[t!]
\centering
\includegraphics[width=0.8\linewidth]{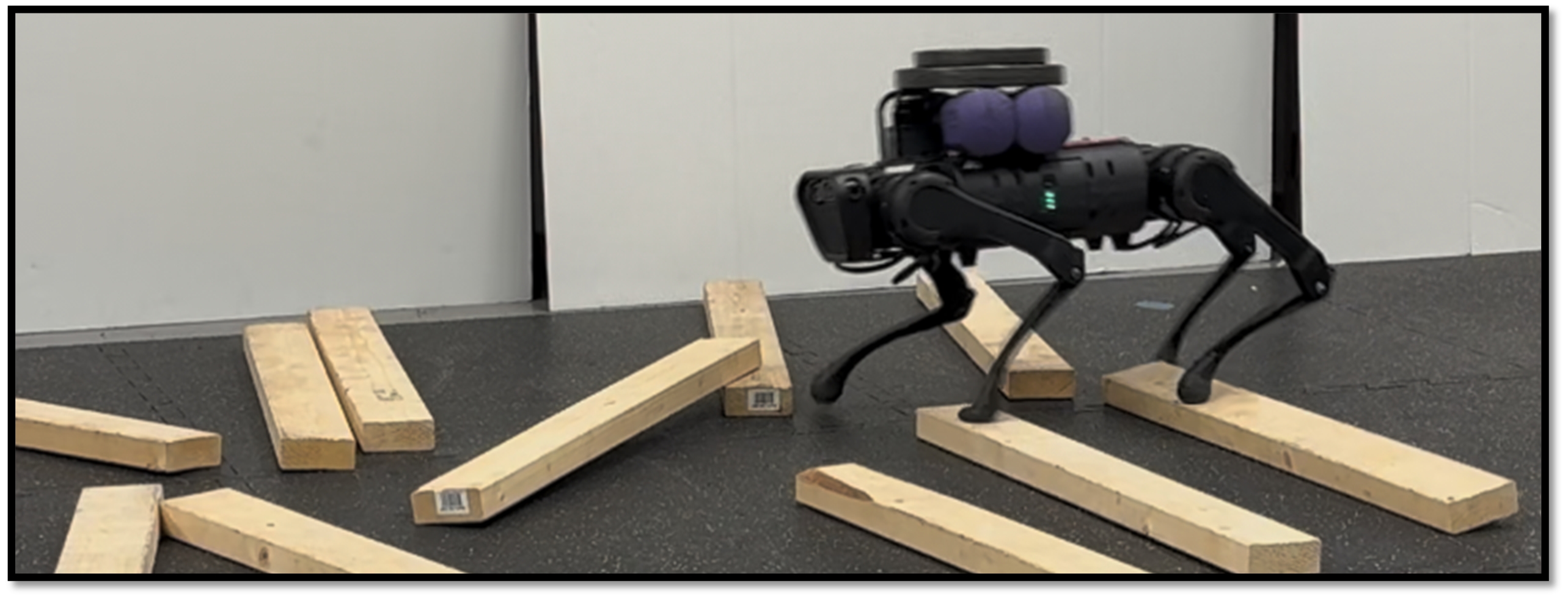}
\vspace{-1.1em}
\caption{Snapshot of the locomotion of the A1 robot with a payload of $11.34$ (kg) ($91\%$ uncertainty) on wooden blocks using the proposed AMPC.}
\label{Fig:MainSnap}
\vspace{-1.5em}
\end{figure}

\begin{figure*}[t!]
\centering
\includegraphics[width=0.9\linewidth]{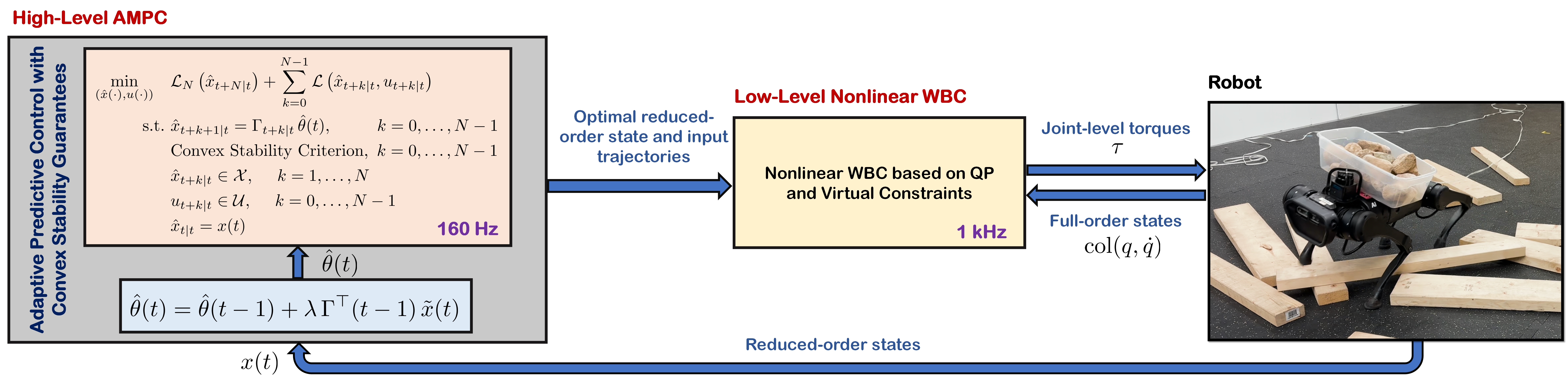}
\vspace{-1.1em}
\caption{Overview of the proposed layered control approach with the AMPC algorithm and the adaptive gradient law at the high level ($160$ Hz) and the nonlinear WBC at the low level ($1$kHz) for robust locomotion with unknown payloads.}
\label{Fig:Overview}
\vspace{-1.5em}
\end{figure*}


\vspace{-1em}
\subsection{Related Work}

Methods for adaptive model-based control of quadrupedal robots, specifically for the task of payload transportation, can be categorized into two groups: those utilizing estimation or identification algorithms and those employing Lyapunov-based adaptive control approaches. In the estimation category, \cite{zhao2021locomotion} presented a reactive-legged locomotion generation scheme that enables adaptation to varying payloads during locomotion, facilitated by a Kalman filter serving as a state estimator. Similarly, \cite{ding2020locomotion} proposed a state estimator estimating the center of mass (COM) components on horizontal planes based on speed errors and contact force differences. Additionally, \cite{tournois2017online} presented recursive approaches for identifying mass and the position of the robot trunk's COM during payload changes using the orientation error and contact forces. Meanwhile, \cite{jin2022unknown} presented an inverse dynamics-based quadratic programming (QP) approach for high payload capacity robot Kirin, incorporating recursive online payload identification.

The second path involves Lyapunov-based adaptive approaches to manage uncertainty in model dynamics while incorporating safety and stability criteria. For example, \cite{minniti2021adaptive} introduced an adaptive control Lyapunov function (CLF)-based MPC that guarantees optimal state-input trajectories meet Lyapunov function criteria. Similarly,  \cite{sombolestan2024adaptive} incorporated adaptive control into a force-based controller to handle static payloads up to $50\%$ of the A1 robot's weight on experimental rough terrain. Moreover, \cite{sun2021online} developed an online learning method for model-based control methods with time-varying and unknown linear residual models, achieving up to $83\%$ payload uncertainty in experiments for the A1 robot's locomotion on nearly flat terrains. Additionally, \cite{bellegarda2022cpg} presented a method for integrating central pattern generators (CPG) into the deep reinforcement learning (RL) framework, enabling robust quadruped locomotion of the A1 robot with dynamically added payloads up to $115\%$ on experimental flat terrains.

In addition to the work previously mentioned, there have been recent advancements in quadrupedal robots specifically designed for heavy payload transportation, such as Barry \cite{valsecchi2023barry} and Kirin \cite{jin2022unknown}. Although these adaptive control approaches have significantly improved quadrupeds' performance in carrying payloads, fundamental questions remain unanswered. This paper aims to address the following \textit{research questions}: 1) How can we develop online and indirect estimation algorithms that ensure the stability of estimation errors for template models? 2) How can we incorporate such stability criteria within a convex MPC framework? For example, the online identification and estimation algorithm in \cite{tournois2017online} does not formally guarantee the asymptotic stability of the estimated error in states. Moreover, the adaptive approach in \cite{jin2022unknown}, formulated in an inverse dynamics-based QP, is designed for quasi-static gaits but not within the MPC framework. Additionally, the AMPC approach in \cite{minniti2021adaptive} addresses stability for adaptive algorithms within the MPC framework using CLFs, but does not cover indirect estimation and identification algorithms. Similarly, \cite{sombolestan2024adaptive} presents an adaptive algorithm by combining an $\mathcal{L}_{1}$ adaptive control with MPC but does not address indirect estimation algorithms. 

In addition to adaptive control approaches, there are robust MPC (RMPC) algorithms designed to handle uncertainties in complex systems. These encompass closed-loop (feedback) min-max MPC \cite{kothare1996robust,min_max_MPC_Scokaert}, open-loop min-max MPC \cite{campo1987robust,zheng1993robust}, and tube MPC \cite{langson2004robust,mayne2005robust}. Stochastic MPC (SMPC), tube-based MPC, and robust convex MPC methodologies have been successfully applied to tackle uncertainty in legged locomotion, leveraging the LIP \cite{villa2017model,gazar2021stochastic}, SRB \cite{xu2023robust}, and centroidal models \cite{SNMPC}. Our previous work in \cite{pandala2022robust} presented an RL-based RMPC algorithm to bridge the gap between reduced- and full-order locomotion models. Additionally, we developed $\mathcal{H}_{2}$- and $\mathcal{H}_{\infty}$-optimal RMPC algorithms to manage uncertainties arising from external forces \cite{pandala_H2_Hinf}. However, the current paper proposes an alternative approach in the context of AMPC algorithms to address the above-mentioned research questions. 


\vspace{-1em}
\subsection{Contributions}

Considering the research question, this paper's primary contribution is the development of an innovative indirect adaptive estimation algorithm integrated with MPC, which provides formal stability guarantees for estimation error. The adaptation law is formulated using gradient descent, and the stability condition is embedded as an inequality constraint within a convex MPC framework. The proposed theoretical work addresses the asymptotic stability of the estimation error in template models.  Additionally, the paper introduces a hierarchical control scheme, where the higher-level AMPC algorithm is applied to an SRB template model with uncertainties in the payload's mass and moment of inertia. The optimal reduced-order and ground reaction force (GRF) trajectories generated by the high-level AMPC algorithm are subsequently passed to a low-level nonlinear whole-body controller (WBC) based on QP and virtual constraints for tracking (see Fig. \ref{Fig:Overview}). 

The paper's theoretical results are validated through extensive numerical simulations and experiments conducted on the A1 robot. Our results demonstrate the robot's capability to transport unmodeled, unknown static payloads up to $132\%$ of its mass in simulations, $109\%$ in experiments on flat terrains, and $91\%$ on rough experimental terrains. To the best of our knowledge, transporting a payload of $11.34$ (kg) (i.e., $91\%$ uncertainty) with the A1 robot and rough terrains, as shown in Fig. \ref{Fig:MainSnap}, has not been reported in alternative model-based control approaches. Specifically, \cite{sombolestan2024adaptive} reported handling $50\%$ payload uncertainty on rough terrains with the same robot, while \cite{bellegarda2022cpg} achieved $115\%$ payload uncertainty on flat terrains using an RL-based controller, but not on rough terrains. Notably, the Unitree’s default model-based controller has a maximum payload rating of $5$ (kg). Comprehensive hardware experiments confirm the method’s efficacy on rough terrains, even with uncertainties such as varying payloads, push disturbances, and obstacles along the robot’s path. Our numerical performance comparisons with a normal MPC and $\mathcal{L}_{1}$ MPC indicate a significant improvement.


\vspace{-1.0em}
\section{High-level AMPC}
\label{sec:ampc}

This section presents the main theoretical results of the paper, introducing the AMPC framework with an indirect adaptive law for steering template models amidst parametric uncertainties caused by the unknown objects the robot transports. It also analyzes the stability properties of the estimation model and embeds them as convex conditions in the AMPC algorithm. Our method aims to estimate the uncertain parameters of the robot-object system, assuming the object's rigid attachment to the robot's trunk. Given this assumption and the negligible leg-to-body mass ratio in quadrupeds, this study utilizes the SRB model. However, the general theoretical results can be extended to other reduced-order models as well. 

\textbf{Template Model:} The equations governing the SRB dynamics can be presented as follows:
\begin{equation}\label{eq:SRB_dyn}
\Sigma:\left\{\begin{aligned}
&\ddot{r} = \frac{f^{\rom{net}}}{m^{\textrm{tot}}} - g_{0}\\
&\dot{R} = R\,\skews(\omega)\\
&\dot{\omega} = I^{-1} \left(R^{\top}\,\tau^{\rom{net}} - \skews(\omega)\,I\,\omega\right),
\end{aligned}
\right.
\end{equation}
where $r\in\Real^{3}$ denotes the COM position in the world frame, $R\in\textrm{SO}(3)$ represents the rotation matrix of the body frame relative to the world frame, $\omega\in\Real^{3}$ denotes the angular velocity in the body frame, $m^{\textrm{tot}}$ represents the \textit{unknown} total mass of the robot-object system, $I\in\Real^{3\times3}$ is the \textit{unknown} moment of inertia in the body frame, and $g_{0}\in\Real^{3}$ denotes the constant gravitational vector. In our notation, $\skews(\cdot):\Real^{3}\rightarrow\mathfrak{so}(3)$ is the skew-symmetric operator. Furthermore, $f^{\textrm{net}}:=\sum_{j\in\mathcal{C}} f_{j}$ and $\tau^{\textrm{net}}:=\sum_{j\in\mathcal{C}} \skews(r_{j})\,f_{j}$ denote the net force and torque generated by the legs, expressed in the world frame. Here, $j\in\mathcal{C}$ is the stance foot index, $\mathcal{C}$ represents the set of contacting feet with the ground, $f_{j}\in\Real^{3}$ is the GRF at the contacting foot $j$, and $r_{j}$ denotes the relative position of the foot $j$ with respect to the COM. 

By linearizing the SRB dynamics \eqref{eq:SRB_dyn} around the current states at discrete time $t\in\Integer:=\{0,1,\cdots\}$ using techniques such as variational-based linearization \cite{Wensing_VBL_HJB},  we obtain the following linear time-varying (LTV) system
\begin{equation}\label{eq:linearized_SRB}
x(t+1)=A(t)\,x(t)+B(t)\,u(t),
\end{equation}
where $x\in\mathcal{X}\subset\Real^{n}$ and $u\in\mathcal{U}\subset\Real^{m}$ denote the reduced-order states and control inputs (i.e., GRFs), respectively, for some positive integers $n$ and $m$. In addition, $\mathcal{X}$ and $\mathcal{U}$ are polyhedra representing the feasible states and admissible set of controls (i.e., linearized friction cone).


\textbf{Problem Statement:} The planning task involves developing an AMPC algorithm with an indirect estimation law to steer \eqref{eq:linearized_SRB} from an initial state to a final one, despite \textit{parametric uncertainties}, arising from the unknown mass properties within the model, which are reflected in the $A$ and $B$ matrices. This must be achieved while satisfying feasibility conditions $x\in\mathcal{X}$ and $u\in\mathcal{U}$ and ensuring a stable estimation error.

To address this, we adopt an adaptive updating law for parameter estimation. We propose using the \textit{gradient descent algorithm}, selected for its proven efficacy and relatively high efficiency compared to alternative adaptive controllers \cite{lavretsky2012robust}. Using this approach, the estimation process is expressed as
\begin{equation}
\hat{\theta}(t+1) = \hat{\theta}(t) - \lambda\, \nabla J(\hat{\theta}),
\label{eq:adaptive_law}
\end{equation}
where $\hat{\theta} \in \Real^p$ denotes the \textit{estimate of the uncertain parameter vector} in the model whose correct value is denoted by $\theta \in \Real^p$ for some positive integer $p$, and $\lambda \in \Real$ stands for a positive adaptive learning rate. Here, $J$ represents a cost function that will be defined later, and $\nabla J := \frac{\partial{J}}{\partial{\hat{\theta}}}^\top$ denotes the gradient of $J$ with respect to $\hat{\theta}$. In our notation, $\hat{(\cdot)}$ and $\tilde{(\cdot)}$ denote the estimated variables and estimated errors, respectively. 


\textbf{Construction of $\theta$ and the Estimation Model:} To enhance the computational efficiency of our model, we estimate only the sections of the Jacobian matrices $A$ and $B$ that contain parametric uncertainties rather than the entire matrices. Closed-form expressions for these matrices, obtained using the variational-based linearization,  are provided in \cite{Wensing_VBL_HJB}. Each element in these matrices can be expressed as the sum of three distinct sections: 1) time-varying deterministic components, 2) time-invariant terms with uncertainties only, and 3) products of deterministic and uncertain components. Constructing $\theta$ involves examining all elements in matrices $A$ and $B$, identifying uncertain portions within each element, and compiling them into $\theta$. This results in a $65$-dimensional $\theta$ vector.

We next reformulate equation \eqref{eq:linearized_SRB} to present a standard linear regression model as the product of the vector $\theta \in \Real^p$ and a regressor matrix $\Gamma(t) \in \Real^{n \times p}$ as follows:
\begin{equation}
    x(t+1) = \Gamma(t)\, \theta.
    \label{eq:regressor_model}
\end{equation}
In this formulation, each element in matrices $A$ and $B$ is expressed as the product of an appropriate horizontal vector and $\theta$, that is, $a_{i,\ell}(t)=h_{i,\ell}^{x}(t)\,\theta$ for all $i,\ell=1,\cdots,n$, where $a_{i,\ell}$ represents the $(i,\ell)$ element of $A$ and $b_{i,\ell}(t)=h_{i,\ell}^{u}(t)\,\theta$ for all $i=1,\cdots,n$ and $\ell=1,\cdots,m$, where $b_{i,\ell}$ denotes the $(i,\ell)$ element of $B$. Here,
$h_{i,\ell}^{x}(t)\in\Real^{1\times p}$ and $h_{i,\ell}^{u}(t)\in\Real^{1\times p}$.
Furthermore, as mentioned previously, some terms within each element might be fully deterministic, necessitating the inclusion of the value ``$1$'' as one of the elements in $\theta$. 

\begin{remark}
The original SRB dynamics in \eqref{eq:SRB_dyn} includes $10$ inertia parameters, which are subject to certain constraints to ensure physical consistency, such as a positive definite inertia matrix. In contrast, the parameter vector $\theta$ is $65$-dimensional. Specifically, these parameters consist of $1$, $\frac{1}{m^{\textrm{tot}}}$, $9$ elements of $I^{-1}$, and $54$ additional elements that can be used to linearly parameterize $I^{-1}\,\skews(\omega)\,I\,\omega$, as given in \eqref{eq:SRB_dyn}. In this work, we do not enforce physical consistency on the estimated parameters. However, the proposed gradient law can be adapted to a projected gradient-descent approach to incorporate physical constraints.
\end{remark}

We now present a closed-form expression for the regressor matrix $\Gamma(t)$. To this end, the $i^{\textrm{th}}$ element of the state vector $x(t+1)$ can be computed as follows:
\begin{equation}
\begin{aligned}
    x_{i}(t+1) &= \sum_{\ell=1}^{n}{h}^x_{i,\ell}(t) \,  \theta  \, x_{\ell}(t) + \sum_{\ell=1}^{m} {h}^u_{i,\ell}(t)  \, \theta  \, u_{\ell}(t)\\
    &=x^{\top}(t) \, h^{x}_{i}(t)  \, \theta+ u^{\top}(t) \,  h^{u}_{i}(t)  \, \theta,
    \label{eq:ith_state_vec}
    \end{aligned}
\end{equation}
where $h_{i}^{x}(t)\in\Real^{n\times p}$ and $h_{i}^{u}(t)\in\Real^{m\times p}$ are defined as $h_{i}^{x}(t):=\sum_{\ell=1}^{n} e_{\ell}^{x}\, h_{i,\ell}^{x}(t)$ and $h_{i}^{u}(t):=\sum_{\ell=1}^{m} e_{\ell}^{u}\, h_{i,\ell}^{u}(t)$ with $\{e_{1}^{x},\cdots,e_{n}^{x}\}$ and $\{e_{1}^{u},\cdots,e_{m}^{u}\}$ representing the standard bases for $\Real^{n}$ and $\Real^{m}$, respectively. Finally, the $i^{\textrm{th}}$ row of the regressor matrix in \eqref{eq:regressor_model} is a function of $(x(t),u(t),t)$, i.e.,
\begin{equation}\label{eq:ith_row_regressor}
\begin{aligned}
    \Gamma_{i}(t) 
    &=\begin{bmatrix}
        x^\top(t) & u^\top(t)
    \end{bmatrix}\begin{bmatrix}
        h_{i}^{x}(t)\\
        h_{i}^{u}(t)
    \end{bmatrix}=:z^\top(t)\,H_{i}(t),
\end{aligned}
\end{equation}
where $z(t):=\col(x(t),u(t))\in\Real^{n+m}$, ``$\col$'' represents the column operator, and $H_{i}(t)\in\Real^{(n+m)\times p}$. Assuming full knowledge of the state and input vectors $(x(t),u(t))$, we can employ the following state estimation model
\begin{equation}\label{eq:state_est}
    \hat{x}(t+1)=\Gamma(t)\,\hat{\theta}(t),
\end{equation}
in which $\hat{\theta}(t)$ represents the estimated parameter according to the gradient law in \eqref{eq:adaptive_law}. We remark that the estimation error dynamics can be expressed as $\tilde{x}(t+1)=\Gamma(t)\,\tilde{\theta}(t)$, where $\tilde{x}(t):=x(t)-\hat{x}(t)$ and $\tilde{\theta}(t):=\theta-\hat{\theta}(t)$ represent the state and parameter estimation errors, respectively. 


\textbf{Cost Function:} With this estimation model, we select the following cost function for the gradient law
\begin{equation}
    J(\hat{\theta}):=\frac{1}{2} \tilde{x}^\top(t+1)\,\tilde{x}(t+1).
\end{equation}
Using the chain rule, we have $\frac{\partial J}{\partial \hat{\theta}} = \frac{\partial J}{\partial \tilde{x}}\,\frac{\partial \tilde{x}}{\partial \tilde{\theta}}\,\frac{\partial \tilde{\theta}}{\partial \hat{\theta}} = -\tilde{x}^\top(t+1)\,\Gamma(t)$, and hence, the gradient update law \eqref{eq:adaptive_law} is expressed as
\begin{equation}\label{eq:gradient_law}
    \hat{\theta}(t+1) = \hat{\theta}(t) + \lambda\, \Gamma^{\top}(t)\,\tilde{x}(t+1).
\end{equation}

\begin{remark}[Real-Time Implementation]
We remark that the gradient update law in \eqref{eq:gradient_law} can be implemented in real time as $\hat{\theta}(t) = \hat{\theta}(t-1) + \lambda\, \Gamma^{\top}(t-1)\,\tilde{x}(t)$, in which $\Gamma(t-1)$ depends on the available past measurements of $(x(t-1),u(t-1))$. Additionally, $\tilde{x}(t)=x(t)-\hat{x}(t)=x(t)-\Gamma(t-1)\,\hat{\theta}(t-1)$. 
\end{remark}


\textbf{Stability Analysis:} Our objective is now to establish sufficient conditions formally guaranteeing the asymptotic stability of the state estimation error $\tilde{x}(t)$. 

\begin{theorem}[Stability Guarantees]\label{Stability_Thm}
Consider the proposed gradient-based update law in \eqref{eq:gradient_law}. If there exists a feasible control input $u(t)\in\mathcal{U}$ that satisfies the following inequality
\begin{equation}
    \lambda_{\max}\left(\Gamma(t) \, \Gamma^{\top}(t)\right) < \frac{2}{\lambda},\quad \forall t\in\Integer,
    \label{eq:NL_stability_gurantees}
\end{equation}
where $\lambda_{\max}(\cdot)$ represents the maximum eigenvalue, the following statements hold.
\begin{enumerate}
    \item The parameter estimation error $\tilde{\theta}(t)$ and the estimated parameter $\hat{\theta}(t)$ are bounded, i.e., $\hat{\theta}(t),\tilde{\theta}(t)\in\mathcal{L}_{\infty}$. 
    \item The state estimation error $\tilde{x}(t)$ is bounded and asymptotically converges to zero, i.e., $\lim_{t\rightarrow\infty}\tilde{x}(t)=0$.
\end{enumerate}
\end{theorem}

\begin{proof}
The time variable is denoted as a subscript throughout the proof. 
Consider the bounded-below function $V_t :=\tilde{\theta}_t^{\top} \,\tilde{\theta}_t\geq0$. We aim to show that it is a decreasing function; thus, $\lim_{t \rightarrow \infty} V_t$ exists. Consequently, $\tilde{\theta}_{t},\hat{\theta}_{t}\in\mathcal{L}_{\infty}$.

To show that $V_t$ is decreasing, we have
\begin{equation}
    \begin{aligned}
    \Delta V_{t} &:= V_{t+1} - V_{t} \\
    &= \left(\theta - \hat{\theta}_{t+1}\right)^{\top}\left(\theta - \hat{\theta}_{t+1}\right)- \left(\theta - \hat{\theta}_{t}\right)^{\top}\left(\theta - \hat{\theta}_{t}\right)\\
    &= - 2\,\theta^{\top} \, \left(\hat{\theta}_{t+1} - \hat{\theta}_{t} \right ) + \hat{\theta}_{t+1}^{\top}  \, \hat{\theta}_{t+1} - \hat{\theta}_{t}^{\top}  \, \hat{\theta}_{t}.
    \end{aligned}
    \label{vdot}
\end{equation}
From equation \eqref{eq:gradient_law}, $\hat{\theta}_{t+1}^{\top} \, \hat{\theta}_{t+1}= \hat{\theta}_{t}^{\top}\, \hat{\theta}_{t} + 2\lambda\,\hat{\theta}_{t}^{\top}\Gamma_{t}^{\top}\tilde{x}_{t+1} + \lambda ^2 \,\tilde{x}_{t+1}^{\top}\Gamma_{t}\,\Gamma_{t}^{\top}\tilde{x}_{t+1}$. Combining this latter equation with equation \eqref{vdot}, we get
\begin{equation*}
    \begin{aligned} 
    \Delta V_{t} &= -2\theta^{\top} (\hat{{\theta}}_{t+1} - \hat{{\theta}}_{t}) + 2\lambda\hat{\theta}_{t}^{\top} \Gamma_{t}^{\top} \tilde{x}_{t+1} +\lambda^2\tilde{x}_{t+1}^{\top} \Gamma _{t} \,\Gamma_{t}^{\top} \tilde{x}_{t+1}\\
    &= -2\lambda \theta^{\top} \Gamma_{t}^{\top} \tilde{x}_{t+1} + 2\lambda \hat{\theta}_{t}^{\top} \Gamma_{t}^{\top} \tilde{x}_{t+1} 
    + \lambda^2 \,\tilde{x}_{t+1}^{\top} \,\Gamma _{t} \Gamma_{t}^{\top}\,\tilde{x}_{t+1}\\
    &=-2\lambda \underbrace{\tilde{\theta}_{t}^{\top}\,\Gamma_{t}^{\top}}_{\tilde{x}^\top_{t+1}}\,\tilde{x}_{t+1} + \lambda^2 \, \tilde{x}_{t+1}^{\top}\, \Gamma _{t} \,\Gamma_{t}^{\top}\,\tilde{x}_{t+1}\\
    &=\tilde{x}_{t+1}^{\top}\, \underbrace{\left(\lambda^2 \,\Gamma_{t}\,\Gamma_{t}^{\top} - 2 \,\lambda \,\identity_{n}\right)}_{=:\Psi_{t}}\,\tilde{x}_{t+1}.
    \end{aligned}
\label{eq:15}
\end{equation*}
If $\Psi_{t}$ is a negative definite matrix, $V_{t}$ becomes decreasing. To have negative definiteness, we should satisfy $\lambda_{\max}(\Psi_{t})<0$, which is equivalent to $\lambda_{\max}(\lambda^2 \,\Gamma_{t}\,\Gamma_{t}^{\top}-2\,\lambda \,\identity_{n}) < 0$ or $\lambda_{\max}(\Gamma_{t}\,\Gamma_{t}^{\top}) < \frac{2}{\lambda}.$
Hence, $\lim_{t\rightarrow \infty}V_{t}$ exists (not necessarily zero), and $\lim_{t\rightarrow \infty}\tilde{\theta}_{t}$ exists. Thus, $\tilde{\theta}_{t},\hat{\theta}_{t} \in \mathcal{L}_{\infty}$. 
Additionally, we have $\Delta V_{t} = \tilde{x}^\top_{t+1} \Psi_{t}\, \tilde{x}_{t+1}$, and hence, $V_{T+1} - V_0 = \sum_{t=0}^{T}\tilde{x}^\top_{t+1} \Psi_{t}\, \tilde{x}_{t+1}$, which results in $V_0 - \lim_{T\rightarrow \infty}V_{T} = -\sum_{t=0}^{\infty}\tilde{x}^\top_{t+1} \Psi_{t}\, \tilde{x}_{t+1}$. Since $-\Psi_{t}$ is positive definite, we have
\begin{equation}
    \begin{aligned}
        &\sum_{t=0}^{\infty}\lambda_{\min}(-\Psi_{t})||\tilde{x}_{t+1}||^2 &&\leq -\sum_{t=0}^{\infty}\tilde{x}^\top_{t+1} \Psi_{t}\,\tilde{x}_{t+1}\\ & &&\leq \sum_{t=0}^{\infty} \lambda_{\max}(-\Psi_{t})||\tilde{x}_{t+1}||^2.
    \end{aligned}
\end{equation}
This latter inequality results in $\sum_{t=0}^{\infty}\lambda_{\min}(-\Psi_{t})||\tilde{x}_{t+1}||^2\leq V_{0}-V_{\infty}$. Finally, by defining $\zeta:=\min_{t\geq0} \lambda_{\min}(-\Psi_{t})>0$, we can conclude that $\sum_{t=0}^{\infty} ||\tilde{x}_{t+1}||^2 \leq \frac{V_0 - V_{\infty}}{\zeta}$ which in turn shows that $\lim_{t \rightarrow \infty} ||\tilde{x}_{t+1}|| = 0$. 
\end{proof}

\textbf{Embedding the Stability Guarantees in a Convex MPC:} Although Theorem \ref{Stability_Thm} provides sufficient conditions for the stability of the state estimation error, condition \eqref{eq:NL_stability_gurantees} cannot be incorporated into a convex MPC formulation. Specifically, the regressor matrix $\Gamma(t)$ is linear with respect to $(x(t), u(t))$, as shown in \eqref{eq:ith_row_regressor}. However, the maximum eigenvalue operator $\lambda_{\max}(\cdot)$ in \eqref{eq:NL_stability_gurantees} is nonlinear and nonsmooth. To address this issue, we present the following lemma, which offers a convex stability condition that can be embedded in a QP-based MPC.

\begin{lemma}[Convex Stability Guarantees]
If there exists a state upper bound $\varepsilon^{x}>0$ and a input $u(t)\in\mathcal{U}$ such that 
\begin{equation}\label{eq:convex_stability_gurantees}
    |u_{j}(t)|\leq \frac{1}{m}\left(\frac{1}{\max_{i}\|H_{i}(t)\|_{1}}\sqrt{\frac{2}{\lambda}} - \varepsilon^{x}\right) \,\,\, \textrm{and} \,\,\, \|x(t)\|_{1}\leq \varepsilon^{x}
\end{equation}
for all $j=1,\cdots,m$ and every $t\in\Integer$, then \eqref{eq:NL_stability_gurantees} is satisfied.
\end{lemma}

\begin{proof}
From the properties of the spectral radius, we have $\lambda_{\max}\left(\Gamma(t)\,\Gamma^\top(t)\right)\leq \|\Gamma(t)\,\Gamma^\top(t)\|\leq \|\Gamma(t)\|^{2}$ for all induced norms. Here, we choose the infinity norm as $\|\Gamma(t)\|_{\infty}=\max_{1\leq i\leq n}\|\Gamma_{i}(t)\|_{1}$, where $\Gamma_{i}(t)$ represents the $i^{\textrm{th}}$ row of $\Gamma(t)$ as given in \eqref{eq:ith_row_regressor}. Using norm properties, $\|\Gamma_{i}(t)\|_{1}\leq\|z(t)\|_{1}\,\|H_{i}(t)\|_{1}=(\|x(t)\|_{1}+\|u(t)\|_{1})\,\|H_{i}(t)\|_{1}$. Replacing \eqref{eq:convex_stability_gurantees} in this latter inequality results in \eqref{eq:NL_stability_gurantees}. 
\end{proof}

\begin{remark}
The convex stability condition in \eqref{eq:convex_stability_gurantees} offers a sufficient criterion for the nonlinear stability condition in \eqref{eq:NL_stability_gurantees} and excludes certain states from the MPC without being overly conservative. In our SRB dynamics experiments (Section \ref{sec::experiments}), the proposed convex AMPC in \eqref{eq:AMPC} showed no infeasibility. For SRB dynamics with $n = 12$ states and $m = 12$ inputs ($z=\col(x,u)\in\Real^{24}$), it is computationally impractical to quantify the coverage of condition \eqref{eq:convex_stability_gurantees} over the feasible region in \eqref{eq:NL_stability_gurantees}  due to $z$'s high dimensionality. However, a simplified scenario with $n = 1$, $m = 1$, and $p = 10$ parameters demonstrates that the condition is not overly restrictive. In this scenario, we randomly generated $1000$ $H_{i}(t)\in\Real^{(n+m)\times p}$ matrices with elements in $[-1,1]$ and discretized the $z$-space $[-10,10]^{2}$ into $200$ grids. With $\lambda=0.2$ and  $\varepsilon^{x}=1$, the condition \eqref{eq:convex_stability_gurantees} forms a rectangle in the $z$-space, covering approximately $40\%$ of the ellipsoid formed by condition \eqref{eq:NL_stability_gurantees}, indicating that it's not overly restrictive. For $n=2$, $m=1$, and $p=50$, the coverage is approximately $39\%$. For higher values of $n$, the computational cost of the coverage is prohibitively intensive.
\end{remark}

\textbf{Convex AMPC:} Now, we are in a position to present the following real-time and convex QP-based AMPC
\begin{equation}
\begin{aligned}
\min_{(\hat{x}(\cdot),u(\cdot))} \quad & 
\mathcal{L}_{N}\left(\hat{x}_{t+N|t}\right) + \sum_{k=0}^{N-1}{\mathcal{L}\left(\hat{x}_{t+k|t}, u_{t+k|t}\right)}\\
\rom{s.t.~} & \hat{x}_{t+k+1|t} = \Gamma_{t+k|t}  \, \hat{\theta}(t), ~~~~~~~~~~~k = 0,\dots,N-1\\
 &\textrm{Convex Stability Criterion \eqref{eq:convex_stability_gurantees}},~k = 0,\dots,N-1\\
   &\hat{x}_{t+k|t} \in \mathcal{X}, ~~~~ k = 1,\dots, N\\
 &  u_{t+k|t}\in \mathcal{U},~~~~ k = 0, \dots, N-1
  \label{eq:AMPC}
    \end{aligned}
\end{equation}
with $\hat{x}_{t|t} = x(t)$, where $N$ is the control horizon, $\mathcal{L}_{N}(x)$ and $\mathcal{L}(x,u)$ represent the terminal and stage costs, respectively, defined as $\mathcal{L}_{N}(\hat{x}_{t+N|t}):=\|\hat{x}_{t+N|t}-x_{t+N|t}^{\des}\|_{P}^{2}$ and $\mathcal{L}(\hat{x}_{t+k|t},u_{k+t|t}):=\|\hat{x}_{t+k|t}-x^{\des}_{t+k|t}\|_{Q}^{2}+\|u_{t+k|t}\|_{R}^{2}$ for some positive definite matrices $P$, $Q$, and $R$ and some desired reference trajectory $x^{\des}(\cdot)$. We remark that the parameter $\hat{\theta}(t)$ is updated when solving the AMPC at time $t$ and remains constant over the control horizon. However, the regressor matrix is updated with virtual time as the predicted estimation $\hat{x}_{t+k|t}$ and the predicted control $u_{t+k|t}$ are updated. Finally, $\hat{x}(\cdot)$ and $u(\cdot)$ are the predicted estimation and control trajectories. We remark that the proposed AMPC in \eqref{eq:AMPC} is applied to the SRB dynamics, not the full-order model of the robot or a kino-dynamic model. Hence, the swing foot trajectories and joint-level torques will be computed using the low-level WBC.


\vspace{-1.5em}
\section{Low-level Nonlinear Controller}
\label{sec::LL}

This section briefly outlines the low-level nonlinear WBC \cite{Fawcett_Hamed_CLF}, employing QP and virtual constraints \cite{Jessy_Book}. The robot is modeled with a floating base, with generalized coordinates represented by $q\in\mathcal{Q}\subset\Real^{n_{q}}$, where $\mathcal{Q}$ denotes the configuration space and $n_{q}$ represents the number of degrees of freedom (DOFs). Joint-level torques are denoted by $\tau\in\mathcal{T}\subset\Real^{m_{\tau}}$, where $\mathcal{T}$ represents allowable torques and $m_{\tau}$ denotes the number of actuators. The full-order model is described by
\begin{equation}\label{Dynamics}
D(q) \,\ddot{q}+H(q,\dot{q})=\Upsilon \,\tau + J^\top(q)\, f,
\end{equation}
where $D(q)\in\Real^{n_{q}\times n_{q}}$ is the mass-inertia matrix, $H(q,\dot{q})\in\Real^{n_{q}}$ represents Coriolis, centrifugal, and gravitational terms, $\Upsilon\in\Real^{n_{q}\times{}m_{\tau}}$ is the input distribution matrix, $J(q)$ is the contact Jacobian matrix, and $f$ denotes the GRFs at the stance feet. Additionally, it is assumed that the positions of the stance leg ends, denoted by $r_{\textrm{st}}$, remain non-slipping, i.e., $\ddot{r}_{\textrm{st}}=0$.

Having established the system dynamics, we present the WBC aimed at tracking the optimal force and reduced-order trajectories generated by the high-level AMPC (see Fig. \ref{Fig:Overview}). Virtual constraints, defined as output functions, are regulated as $h(q,t):=h_{0}(q)-h^{\des}(t)$, and enforced via feedback linearization \cite{Isidori_Book}. Here, $h_{0}(q)$ comprises controlled variables, including the COM position, trunk's orientation, and Cartesian coordinates of swing leg ends, while $h^{\des}(t)$ represents the desired evolution of $h_{0}(q)$. Our approach determines swing leg end positions using Raibert's heuristic \cite[Eq. (2.4)]{raibert1986legged}, and swing leg trajectories are defined using an $8$th-order B\'ezier polynomial connecting the current footholds to the upcoming ones. These virtual constraints are combined into a strictly convex QP, solved at 1kHz as follows \cite{Fawcett_Hamed_CLF}, to compute torques
\begin{alignat}{4}\label{QP}
    &\min_{(\tau,f,\delta)} \,\,\,&& \frac{\gamma_{1}}{2}\|\tau\|^2 + \frac{\gamma_{2}}{2}\|f - f^{\des}\|^2 + \frac{\gamma_{3}}{2} && \|\delta\|^2 \nonumber\\
    &\textrm{s.t.} && \ddot{h}(\tau,f) = -K_{P}\,h - K_{D}\,\dot{h} + \delta && \textrm{(Output Dynamics)}\nonumber\\
    & && \ddot{r}_{\textrm{st}}(\tau,f) = 0 \nonumber && \textrm{(No slippage)}\\
    & && \tau\in\mathcal{T},\quad f\in\mathcal{FC} && \textrm{(Feasibility)},
\end{alignat}
where $\gamma_{1}$, $\gamma_{2}$, and $\gamma_{3}$ are positive weighting factors, and the desired force profile $f^{\des}(t)$ is dictated by the high-level AMPC. The equality constraints encompass the desired output dynamics $\ddot{h} + K_{D},\dot{h} + K_{P},h = \delta$ with positive definite matrices $K_{P}$ and $K_{D}$, and $\delta$ as a defect variable to ensure the feasibility of the QP, alongside the no-slippage condition on acceleration. Notably, $\ddot{h}$ and $\ddot{r}_{\textrm{st}}$ are affine functions of $(\tau,f)$ and can be calculated using Lie derivatives \cite[Appendix A]{Jeeseop_Hamed_ASME}. The inequality constraints address the feasibility of torques and GRFs with $\mathcal{FC}$ being the linearized friction cone. 


\vspace{-1em}
\section{Experiments}
\label{sec::experiments}

\begin{figure*}[t!]
\vspace{0.1em}
\centering
\includegraphics[width=0.9\linewidth]{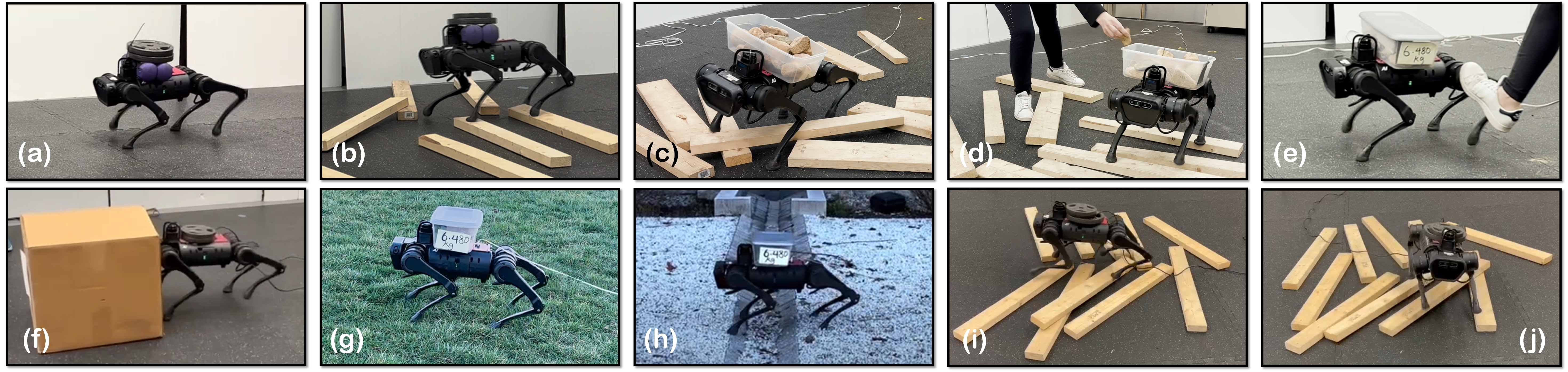}
\vspace{-1.2em}
\caption{Snapshots of experiments for (a) flat terrain locomotion with a $13.6$ (kg) payload ($109\%$ uncertainty); (b) and (c) rough terrain locomotion with wooden blocks and $11.34$ (kg) ($91\%$ uncertainty) and $8.2$ (kg) payloads; (d) dynamic payload addition; (e) pushing the robot with a $6.48$ (kg) payload; (f), pushing a $3.7$ (kg) object while carrying a $4.5$ (kg) payload; (g) and (h) locomotion on grass and gravel with a $6.48$ (kg) payload; (i) and (j) backward and lateral trot on rough terrain with $9$ (kg) and $6.8$ (kg) payloads and a COM offset. In all of these experiments, the desired forward velocities are $0.5$ (m/s), except for (b) and (d), which are $0.4$ (m/s) and $0.6$ (m/s), respectively. For the backward and lateral trot, the desired velocity is $0.4$ (m/s) and $0.15$ (m/s).}
\label{Fig:Snapshots}
\vspace{-1.5em}
\end{figure*}

This section presents the results of our numerical simulations and hardware experiments, offering a comprehensive evaluation of the proposed control algorithm's effectiveness. 


\vspace{-1em}
\subsection{Controller Synthesis}

We conduct our tests using the A1 quadrupedal robot from Unitree Robotics, featuring $n_{q}=18$ DOFs and $m_{\tau}=12$ actuators. It has a weight of $12.45$ (kg) and stands at a height of $0.26$ (m). The robot's $18$ DOFs include six DOFs for unactuated position and orientation of the body and $12$ DOFs for actuated leg joints—each leg consisting of 2-DOF hip and 1-DOF knee joints. The floating-base states of the robot are directly accessible through the robot's hardware and IMU, with the exception of the COM positions, which are determined using a kinematic estimator. The control horizon for the AMPC problem is taken as $N = 7$ discrete sample times with the sampling time of $T_{s}=6.25$ (ms). The AMPC parameters are chosen as $Q=\textrm{diag}\{Q_{r},Q_{\dot{r}},Q_{\xi},Q_{\omega}\}$ with $Q_{r}=\textrm{diag}\{1e5,2e5,1e6\}$, $Q_{\dot{r}}=\textrm{diag}\{1e5,1e5,1e5\}$, $Q_{\xi}=\textrm{diag}\{1e3,1e3,1e3\}$, $Q_{\omega}=\textrm{diag}\{5e3,5e3,5e3\}$, $P=10\,Q$, and $R=\identity$. As for the adaptive updating law, we settle on the learning rate of $\lambda = 0.2$ for all cases. The friction coefficient is also assumed to be $\mu=0.6$. The low-level WBC employs weighting factors of $\gamma_1 = 1$, $\gamma_2 = 2e3$, and $\gamma_3 = 9e5$ with the PD gains of $K_{P}=400$ and $K_{D}=40$. 

The high-level AMPC and low-level QP are solved online at rates of $160$ Hz and $1$ kHz, respectively, using qpSWIFT \cite{qpSWIFT}, on an Intel Xeon W-2125 processor. Computation times are approximately $0.95$ (ms) and $0.21$ (ms) for the AMPC and low-level controller, respectively. Finally, we utilize the RaiSim physics engine \cite{RAISIM} for numerical simulations. 


\vspace{-0.8em}
\subsection{Numerical Simulations and Hardware Experiments}

In our experiments, we evaluated the effectiveness of the proposed adaptive control algorithm on the robot carrying \textit{unknown} payloads on its trunk across various environments, ranging from flat terrains to challenging landscapes with obstacles in the path and varying ground heights, as shown in Fig. \ref{Fig:Snapshots}. The hardware tests included terrains such as wooden blocks, grass, and loose gravel. We also considered both static and dynamic (i.e., time-varying) payloads. The robot successfully carried a $6.48$ (kg) payload across all of these terrains at a $0.5$ (m/s) speed. Videos of experiments are available online \cite{YouTube_AMPC}. 

\textbf{Static Payloads:} To evaluate the robustness of the proposed approach, simulations were conducted to reach a maximum payload capacity of $16.5$ (kg), equivalent to $132\%$ of the robot's mass. This surpasses the normal MPC's capacity of $6.5$ (kg). In hardware experiments, the robot successfully carried $13.6$ (kg) ($109\%$ uncertainty) on flat terrain at a speed of $0.5$ (m/s) and $11.34$ (kg) ($91\%$ uncertainty) on rough terrain at a speed of $0.4$ (m/s) (see Fig. \ref{Fig:Snapshots}(a)-(c)). To the best of our knowledge, achieving a maximum payload capacity with $91\%$ uncertainty for the rough terrain locomotion of the A1 robot has not been reported in other model-based control approaches. In light of pushing unknown obstacles, our algorithm enabled the robot to steadily handle an obstacle weighing $3.7$ (kg) while carrying a $4.5$ (kg) payload, as illustrated in Fig. \ref{Fig:Snapshots}(f). The robot could also traverse the rough terrain in a backward and lateral manner and subject to a payload of $9$ (kg) and $6.8$ (kg), respectively, with a COM offset (see Figs. \ref{Fig:Snapshots}(i) and \ref{Fig:Snapshots}(j)).

\textbf{Dynamic Payloads:} In dynamic payload scenario simulations, where objects are thrown on the robot's trunk at random time samples, the robot maintained robust stability on rough terrains, handling up to $9.804$ (kg) ($79\%$ of its mass) with AMPC, compared to $5.102$ (kg) ($41\%$ of its mass) with normal MPC. The robot also demonstrated stability in hardware experiments with dynamic payload additions across different terrains, as shown in Fig. \ref{Fig:Snapshots}(d). In this experiment, the robot's initial payload was $6.48$ (kg), and additional payloads increased it to $9.12$ (kg) ($73\%$ uncertainty) while walking on wooden blocks. Additionally, both hardware experiments and simulations revealed that the robot could withstand external disturbances, such as pushes, as illustrated in Fig. \ref{Fig:Snapshots}(e).

\textbf{Quantitative Analysis and Comparison:} To quantitatively demonstrate the AMPC algorithm's effectiveness, Figs. \ref{Fig:data_uneven_terrain} and \ref{Fig:data_push} show the optimal reduced-order states and GRFs prescribed by the AMPC, along with the estimated mass parameter, in two hardware experiments depicted in Figs. \ref{Fig:Snapshots}(b) and \ref{Fig:Snapshots}(d). In these experiments, the roll, pitch, and y-direction velocities are commanded to be zero,  the x-direction velocity is set at $0.4$ (m/s) in Fig. \ref{Fig:Snapshots}(b) and $0.6$ (m/s) in Fig. \ref{Fig:Snapshots}(d), and the robot height is commanded at $0.26$ (m). The results clearly indicate that the prescribed states closely follow the desired trajectories, and the mass estimation, starting from the known robot mass of $12.45$ (kg), gets updates and remains within a bounded neighborhood of the true mass. 

\begin{figure}[t!]
\vspace{0.1em}
\centering
\includegraphics[width=0.97\linewidth, trim={2cm 0.5cm 1cm 0.75cm}, clip]{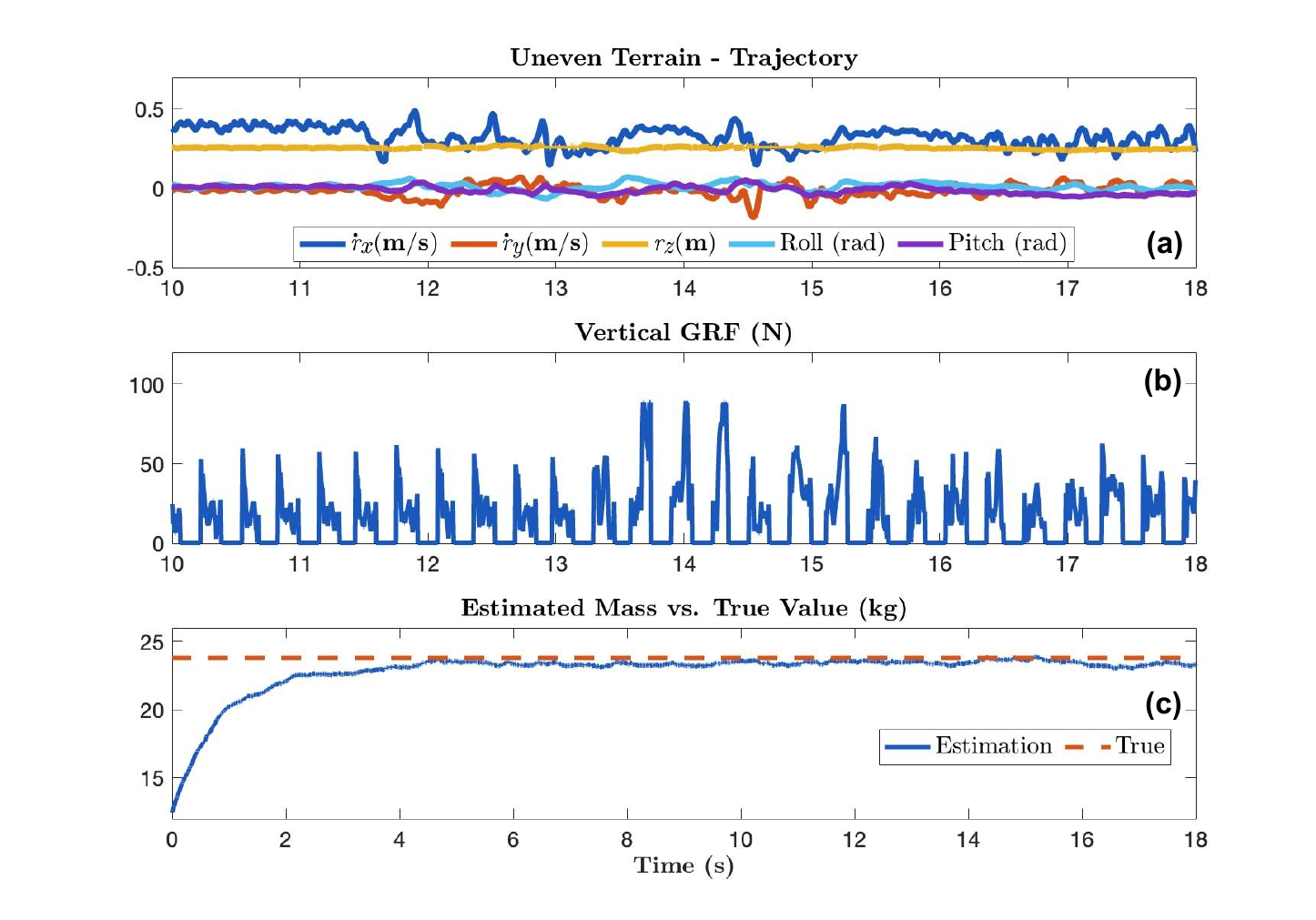}
\vspace{-1.3em}
\caption{Plot of the prescribed reduced-order trajectory and vertical GRF for the front right leg by the AMPC planner at the forward speed of $0.4$ (m/s) in (a) and (b). (c) Evolution of the estimated mass parameter. Despite the challenges posed by uneven terrain and an unknown payload of $11.34$ (kg), the robot demonstrates robust locomotion, as depicted in Fig. \ref{Fig:Snapshots}(b). }
\label{Fig:data_uneven_terrain}
\vspace{-2.0em}
\end{figure}

\begin{figure}[t!]
\centering
\includegraphics[width=0.97\linewidth, trim={2cm 0.5cm 1cm 0.75cm}, clip]{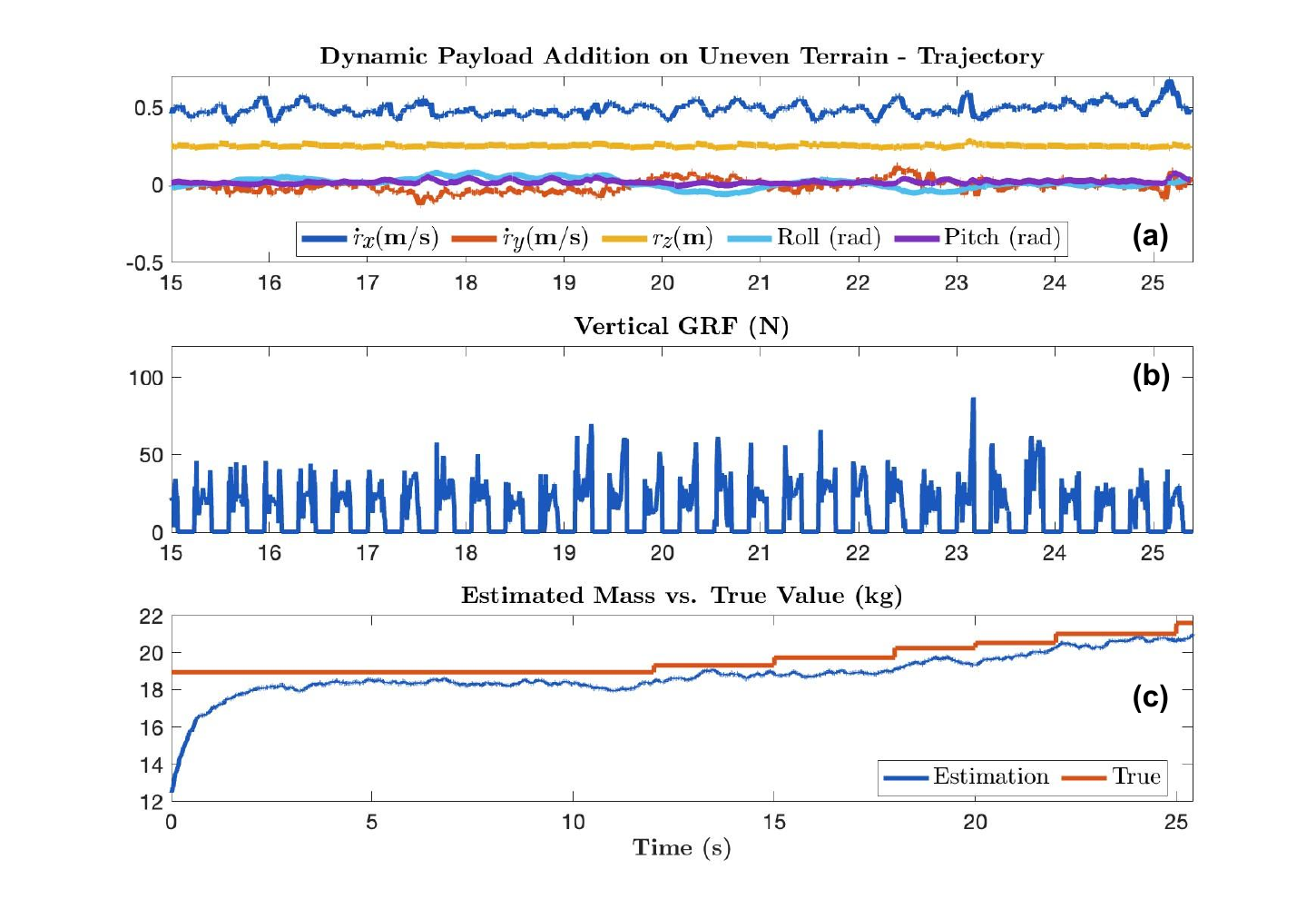}
\vspace{-1.4em}
\caption{Plot of the prescribed reduced-order trajectory and vertical GRF for the front right leg by the AMPC planner at the forward speed of $0.6$ (m/s) in (a) and (b). (c) Evolution of the estimated mass parameter. The robot is subject to carrying a payload of $6.48$ (kg) and dynamic additions, as shown in Fig. \ref{Fig:Snapshots}(d). The true mass varies in response to these dynamic payload additions.}
\label{Fig:data_push}
\vspace{-1.0em}
\end{figure}

To quantitatively study the efficacy of the proposed AMPC algorithm using the full-order model, we conducted extensive numerical simulations to compare the performance of the AMPC and normal MPC for trotting at a speed of $0.5$ (m/s) while carrying a $6.5$ (kg) payload across $1500$ randomly generated uneven terrains. Specifically, we generated $1500$ heightmaps, each featuring a different distribution of blocks. These blocks, each $5$ (cm) in height, were randomly distributed over a terrain $30$ (m) in length (approximately $85$ times the robot’s body length). Figure \ref{Fig:comparison_AMPC_vs_MPC} illustrates the success rates of the AMPC algorithm and the normal MPC on these randomly generated terrains over the traveled distance. Here, the normal MPC employs the same low-level nonlinear WBC. As shown in Fig. \ref{Fig:comparison_AMPC_vs_MPC}, the overall success rates of AMPC and MPC across the entire terrain are $88.22\%$ and $18.89\%$, respectively. 

To compare the trajectory tracking performance, we conducted numerical simulations using payloads of $4$ (kg), $6$ (kg), and $10$ (kg). These simulations were performed over $1500$ previously generated uneven terrains, each with a length of $30$ (m), using both the AMPC and the normal MPC methods. Table \ref{Tab:1} presents the robot's average actual speed and height. In all tests, the commanded velocity and COM height were set to $0.5$ (m/s) and $0.26$ (m), respectively. The results indicate that the nominal MPC fails to handle the $10$ (kg) payload. 


\begin{figure}[t!]
\centering
\includegraphics[width=0.95\linewidth, trim={0 0 0 3cm}, clip]{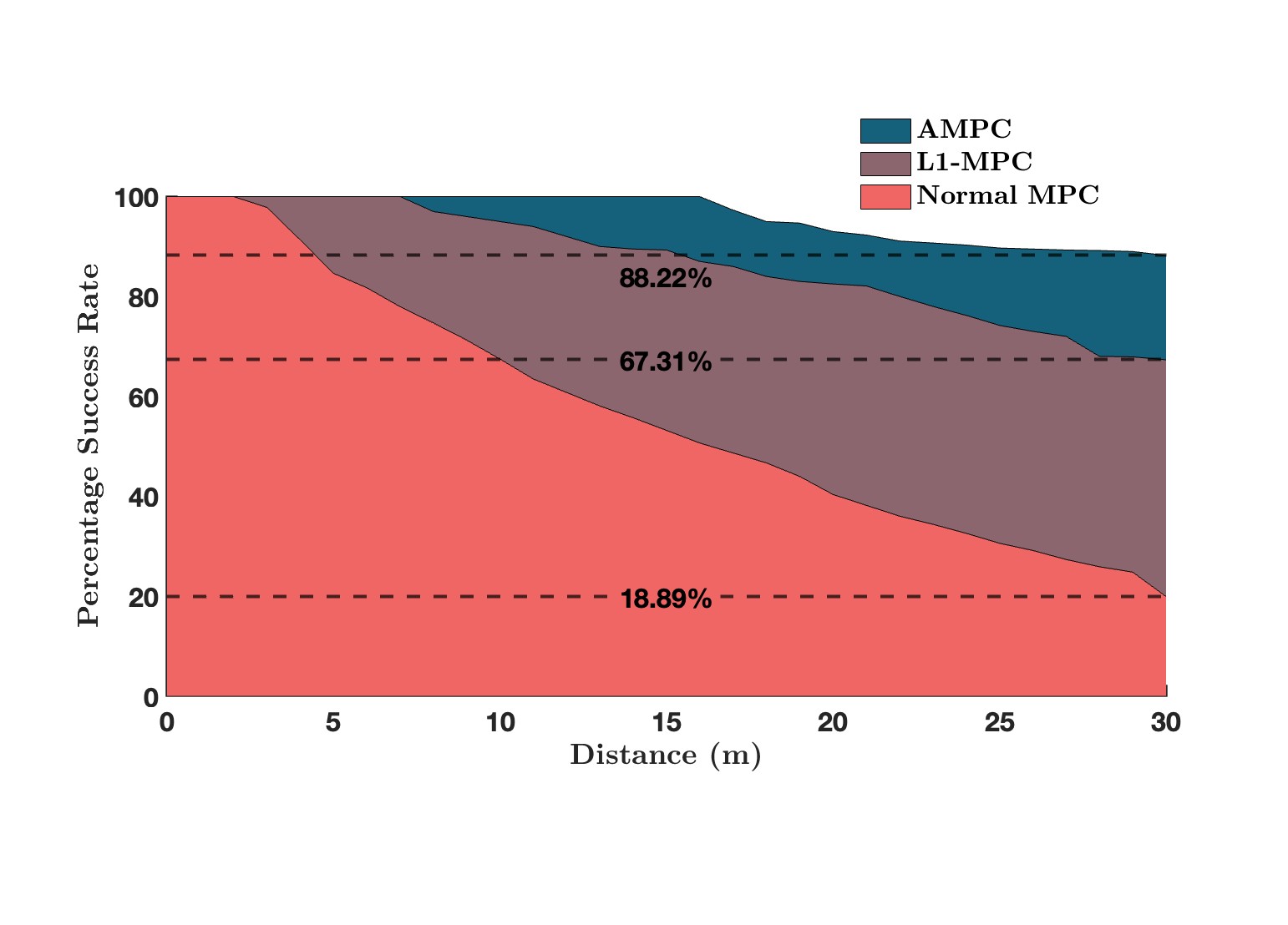}
\vspace{-4.0em}
\caption{The success rates of various predictive control algorithms, including the AMPC, normal MPC, and $\mathcal{L}_{1}$ MPC with a $6.5$ (kg) payload on $1500$ randomly generated rough terrains versus the distance.} 
\label{Fig:comparison_AMPC_vs_MPC}
\vspace{-2.0em}
\end{figure}

\begin{table}[t]
\caption{Trajectory tracking performance over $1500$ uneven terrains.}  
\vspace{-1.5em}
\small 
\begin{center}
\setlength{\fboxsep}{0pt} 
\fbox{
\begin{tabular}{|c|c|c|c|c|}
\hline
\multicolumn{1}{|c|}{Payload mass (kg)} & \multicolumn{2}{c|}{AMPC} & \multicolumn{2}{c|}{MPC} \\
\cline{2-5}
 & $\Bar{\dot{r}}_x$ (m/s) & $\Bar{r}_z$ (m) & $\Bar{\dot{r}}_x$ (m/s) & $\Bar{r}_z$ (m) \\
\hline
$4$ & $0.493$ & $0.252$ & $0.489$ & $0.251$ \\
$6$ & $0.486$ & $0.246$ & $0.467$ & $0.239$ \\
$10$ & $0.472$ & $0.235$ & X & X \\
\hline
\end{tabular}
\vspace{-0.5em}
}
\end{center}
\label{Tab:1}
\vspace{-2.5em}
\end{table}


\vspace{-1em}
\subsection{Comparison with an $\mathcal{L}_{1}$ Adaptive MPC}

To evaluate the efficacy of the proposed gradient-descent-based AMPC approach in comparison to other adaptive predictive control techniques, it is important to note that the method in \cite{minniti2021adaptive} assumes parametric uncertainty in a matched form. This implies that the uncertainty affects the model through the same channel as the inputs, specifically the GRFs. As emphasized in \cite{minniti2021adaptive}, this assumption necessitates that the robot maintains contact with at least three feet or that the uncertainty influences a controllable direction during trotting gaits. Our proposed approach, however, does not need the matched form assumption to address uncertainties. The method presented in \cite{sombolestan2024adaptive} integrates $\mathcal{L}_{1}$ adaptive control with MPC. Notably, MPC is employed for both the reference model and the actual system within the adaptive framework, leading to significant computational demands. To mitigate this, \cite{sombolestan2024adaptive} introduces a frequency update scheme to optimize the allocation of processing resources across control components. In contrast, our proposed approach is built on a single and computationally efficient AMPC algorithm.

Given that the approach in \cite{sombolestan2024adaptive} is based on $\mathcal{L}_{1}$ adaptive control with two MPCs, we quantitatively evaluate the effectiveness of our proposed gradient-descent-based AMPC algorithm in comparison to a simplified and general $\mathcal{L}_{1}$ MPC, adapted from \cite{L1_AMPC_Quadrotor} and extended for application to quadrupedal robots. In this framework, the output of a nominal QP-based MPC is integrated with an $\mathcal{L}_{1}$ adaptive controller to estimate uncertainties in the SRB dynamics, utilizing a piecewise constant adaptation law similar to the one described in \cite[Sec. III. D]{L1_AMPC_Quadrotor}. We then conduct extensive numerical simulations to verify the robustness of the closed-loop systems with the proposed AMPC approach and the $\mathcal{L}_{1}$ MPC during trotting at a speed of $0.5$ (m/s) while carrying a $6.5$ (kg) payload across $1500$ randomly generated uneven terrains. Figure \ref{Fig:comparison_AMPC_vs_MPC} illustrates the percentage success rate of two adaptive MPC methods across randomly generated heightmaps over the traveled distance. The overall success rates of AMPC and $\mathcal{L}_{1}$-MPC across the entire length of the terrain are $88.22\%$ and $67.31\%$, respectively.


\vspace{-1em}
\section{Conclusion and Future Work}
\label{sec::conclusion}

This paper introduced a hierarchical adaptive planning and control algorithm for robust payload transportation by quadrupedal robots on rough terrains. The algorithm integrates an indirect adaptive estimation law based on a gradient descent method with a convex MPC. The study formally addressed the asymptotic stability properties of the state estimation and presented sufficient stability conditions that can be embedded as convex inequality constraints within the MPC framework. Additionally, it proposed a layered structure where the AMPC algorithm is applied to the SRB template dynamics at the higher level of the hierarchy, which addresses unknown and varying payloads. The optimal reduced-order state and input trajectories prescribed by the AMCP are then utilized by a low-level nonlinear WBC based on QP and virtual constraints for tracking. The theoretical results were validated through extensive numerical simulations and experiments conducted on the A1 robot. Our experimental results demonstrate the robot's capability to transport unmodeled and unknown payloads up to $109\%$ on flat terrains and $91\%$ on challenging rough experimental terrains, pushing the robot to its limits. The comprehensive hardware experiments also confirmed the method's efficacy on rough terrains, even in the presence of uncertainties such as varying payloads and push disturbances. 

The convex AMPC algorithm introduced in this study relies on a linearized template model and does not account for the switching of stance legs throughout the control horizon. Future research will focus on extending this approach to nonlinear AMPC frameworks that incorporate leg switching over the control horizon. This extension aims to rigorously evaluate the stability of gait patterns for hybrid locomotion models.


\vspace{-0.8em}
\bibliographystyle{IEEEtran}
\bibliography{references}

\end{document}